\documentclass{article}
\PassOptionsToPackage{numbers, compress}{natbib}

\usepackage[preprint]{neurips_2023}

\usepackage[utf8]{inputenc} 
\usepackage[T1]{fontenc}    
\usepackage{hyperref}       
\usepackage{multirow}
\usepackage{url}            
\usepackage{booktabs}       
\usepackage{amsfonts}       
\usepackage{amsmath}       
\usepackage{amsthm}
\usepackage{nicefrac}       
\usepackage{microtype}      
\usepackage{xcolor}         
\usepackage{csquotes}
\usepackage{graphicx}
\usepackage{wrapfig}
\newtheorem{theorem}
{Theorem}
\newtheorem{thm}{Theorem}
\newtheorem{lem}[thm]{Lemma}
\newtheorem*{lem_1}{Lemma 1}
\newtheorem*{thm_1}{Theorem 1}

\usepackage[capitalize]{cleveref}
\crefname{section}{Sec.}{Secs.}
\Crefname{section}{Section}{Sections}
\Crefname{table}{Table}{Tables}
\crefname{table}{Tab.}{Tabs.}

\title{No Free Lunch: The Hazards of Over-Expressive Representations in Anomaly Detection}

\author{Tal Reiss  $\quad$  Niv Cohen $\quad$ Yedid Hoshen\\
School of Computer Science and Engineering\\
The Hebrew University of Jerusalem, Israel\\ 
{\tt\small tal.reiss@mail.huji.ac.il}\\
}

\begin{document}

\maketitle

\begin{abstract}
Anomaly detection methods, powered by deep learning, have recently been making significant progress, mostly due to improved representations. It is tempting to hypothesize that anomaly detection can improve indefinitely by increasing the scale of our networks, making their representations more expressive. In this paper, we provide theoretical and empirical evidence to the contrary. In fact, we empirically show cases where very expressive representations fail to detect even simple anomalies when evaluated beyond the well-studied object-centric datasets. To investigate this phenomenon, we begin by introducing a novel theoretical toy model for anomaly detection performance. The model uncovers a fundamental trade-off between representation sufficiency and over-expressivity. It provides evidence for a no-free-lunch theorem in anomaly detection stating that increasing representation expressivity will eventually result in performance degradation. Instead, guidance must be provided to focus the representation on the attributes relevant to the anomalies of interest. We conduct an extensive empirical investigation demonstrating that state-of-the-art representations often suffer from over-expressivity, failing to detect many types of anomalies. Our investigation demonstrates how this over-expressivity impairs image anomaly detection in practical settings. We conclude with future directions for mitigating this issue.
\end{abstract}

\section{Introduction}
Anomaly detection, identifying patterns that deviate significantly from the norm, is a crucial problem for science and industry. It has many applications ranging from identifying fraudulent transactions in financial systems to scientific phenomena that deviate from the existing paradigm (e.g. supernovae, quantum physics). Different from supervised learning, which automates known processes, anomaly detection seeks to discover the unknown. This is a more challenging task but with potentially very significant benefits.
New scientific and industrial discoveries are crucial for tackling the challenges facing humanity and may benefit from better anomaly detection capabilities. This makes advancing anomaly detection methods a high-priority task.

The introduction of deep networks has resulted in significant improvements in the state-of-the-art performance on anomaly detection benchmarks, for images \cite{panda_,patchcore}, videos \cite{multi_task_vad,ai_vad}, time-series \cite{time_series} and other modalities. This begs the question:

\begin{displayquote}
\textit{Can scaling the current approaches lead to universal anomaly detection?}
\end{displayquote}

We initiate our investigation by conducting a preliminary experiment to examine the impact of increasing representation expressivity on anomaly detection performance. As we gradually enhance the expressivity of the representation, we reach a critical point where irrelevant attributes are incorporated into the representation. This is without contributing to detection performance improvement. To better understand this phenomenon, we present a model that shows how discovering some anomalies necessarily comes at the expense of missing others. For example, consider an explorer inspecting stones in a mine. Many of the stones are anomalous in different ways, one stone is the largest, another stone can be the hardest while a third is the softest. Intuitively, if we wish to order the stones from the most anomalous to the least anomalous one, we have a trade-off: the more attributes we may wish to account for, the less sensitive the detector becomes to anomalies in each individual attribute.

We propose a simple mathematical model for the trade-off demonstrated by the above motivating example. Let the sample (e.g., a stone) be described by a set of $d$ attributes (e.g., hardness, size, weight). Further assume that the important anomaly (e.g., hardness) has an unusual value of the first attribute, with all other attributes having normal values (e.g., a very hard stone of weight and size drawn from the distribution of the normal data). This makes $d-1$ attributes irrelevant for the purposes of detecting the important anomaly. We show that increasing the number of (irrelevant) attributes makes the method significantly less sensitive. Specifically, for a fixed false positive detection rate, the true positive rate of the anomaly classification converges to the false positive detection rate as we increase $d$ at a rate of $\frac{1}{\sqrt{d}}$. The above model reveals a simple bias-variance trade-off:  On the one hand, the representation must include the attribute relevant to separating normal and anomalous data among its $d$ attributes. As this attribute is not known in advance, this encourages large representations that are expressive of all possible attributes. On the other hand, including irrelevant attributes reduces the sensitivity of the method and leads to a convergence towards a random classifier; encouraging minimal representations.

Navigating the above trade-off requires using a prior over the attributes most likely to lead to important anomalies. 
There is however an inherent conflict between priors and anomaly discovery; anomalies are unexpected and counter-intuitive and are likely to defy priors.
While most existing methods do not explicitly acknowledge using a prior, we show that the leading approaches are strongly biased toward anomalies in a particular set of attributes (mostly novel object category or deformations). While this is advantageous in several cases, including current academic benchmarks, we show these implicit priors fail when they are mismatched with the desired anomaly. Similar conclusions follow for the out-of-distribution (OOD) detection and outlier exposure (OE) settings. 

Our main contributions are:
\begin{enumerate}
\item Identifying over-expressivity as a bottleneck in scaling deep anomaly detection.
\item Proposing a toy model for quantifying the bias-variance trade-off resulting in a "no free lunch" (NFL) principle for anomaly detection.
\item Empirical analysis demonstrating and quantifying how current methods are affected by the above NFL principle.
\end{enumerate}

\section{Related Work}
Anomaly detection has been extensively studied in the machine learning community, and a wide range of methods have been proposed over the years. Here, we provide an overview of relevant works, focusing on representation learning in anomaly detection, challenges associated with over-expressivity, and utilizing guidance to detect anomalies.

\textbf{Representation learning for anomaly detection.} Anomaly detection has traditionally relied on classical approaches such as density estimation \cite{eskin2002geometric, latecki2007outlier} or reconstruction \cite{pca}. However, deep learning has revolutionized anomaly detection by incorporating deep representations into these classical methods \cite{mathieu2015deep, zhang2016colorful, larsson2016learning, noroozi2016unsupervised}. Deep representation learning for anomaly detection often involves self-supervised learning techniques. One common approach is the use of autoencoders \cite{deepsvdd}, which learn to reconstruct input data by compressing it into a low-dimensional representation and then decoding it back to its original form. Another self-supervised method is rotation classification \cite{geom, mhrot}, where the model learns to classify the correct rotation of an input image, forcing it to capture useful visual features. Contrastive learning methods \cite{csi, droc,density4_eccv22_contrastive,contrative_ssl2_eccv22,density2_cvpr23_contrastive} have also been employed. These methods aim to learn representations that maximize the similarity between augmented versions of the same image while minimizing similarity to other images. Another representation learning paradigm for anomaly detection is to combine pretrained representations with anomaly scoring functions \cite{perera2019learning, panda_,patchcore,mean_shifted,transformaly,pretrained1_cvpr23,pretrained2_cvpr23,pretrained3_cvpr23}. This involves using representations pretrained on large-scale external dataset, such as ImageNet classification, and then fine-tuning them for the specific anomaly detection task using the provided normal samples in the training set. These methods achieve high performance in academic anomaly detection datasets.

\textbf{Over-expressivity.} The concept of over-expressivity has been discussed in the general machine learning literature \cite{iml_book}, but its specific implications for anomaly detection have been relatively overlooked. While anomaly detection deep learning approaches have achieved impressive performance on academic anomaly detection datasets, recent studies have revealed limitations in detecting anomalies beyond well-studied object-centric datasets. It has been observed \cite{lars} that even highly expressive representations fail to identify simple anomalies when the anomalies lie in unexplored attribute spaces. While the paper suggested that using the Mahalanobis distance (essentially a Gaussian density model) might mitigate this issue, our analysis shows that issues remain even under such models.

\textbf{Guidance in anomaly detection.} Guidance plays a crucial role in anomaly detection, as prior knowledge is often required to effectively identify unexpected patterns. Approaches such as PANDA \cite{panda_} and PatchCore \cite{patchcore} utilize pre-trained representations, incorporating implicit priors from ImageNet or CLIP \cite{pretrained3_cvpr23,ai_vad} to enhance anomaly detection. Out-of-distribution (OOD) detection methods \cite{hendrycks_ood,ood_balaji,galil2023framework} assume anomalies have unusual values in pre-specified attributes, while outlier exposure methods \cite{oe, mirzaei2022fake} rely on external datasets resembling anomalies. Other methods \cite{zhang2023robustness,red_panda} allow users to specify irrelevant attributes. Similarly, self-supervised methods \cite{mhrot,csi,droc,ssl1_eccv22,contrative_ssl2_eccv22,density4_eccv22_contrastive,density2_cvpr23_contrastive} remove nuisance factors from the representation using augmentation. In this paper, we analyze guidance through our developed framework, exploring its impact on the trade-off between representation expressivity and anomaly detection performance (\cref{sec:guidance_analysis}).

\section{A Motivating Example - Trade-off between Expressivity and Representation Sufficiency}

\begin{wrapfigure}{r}{0.4\textwidth}
\vspace{-1.5em}
    \centering
    \includegraphics[width=0.4\textwidth]{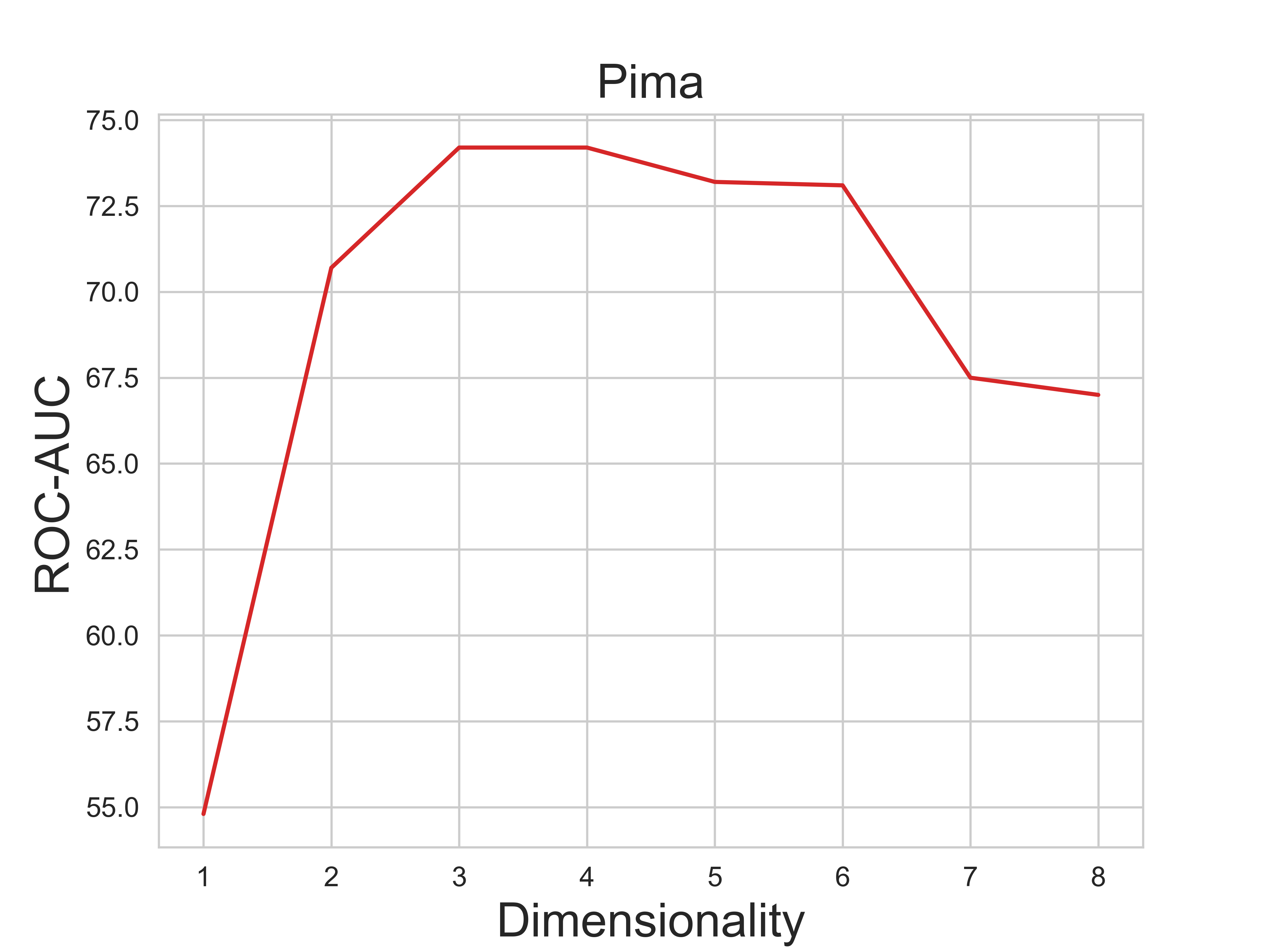} 
	\caption{Anomaly detection performance on the Pima dataset as the dimensionality of the representation increases.}
	\label{fig:motivating_example}
\end{wrapfigure}

We aim to elucidate the trade-off between over-expressivity and the sufficiency of representations in anomaly detection. In this section, we motivate this research by considering a simple, tabular anomaly detection task. The established Pima dataset (From ODDS \cite{odds}), contains biomarker attributes for a set of patients including pregnancies, glucose levels, blood diabetes pedigree, insulin, and age. Each patient is also labeled as healthy or having diabetes. The Pima dataset is converted to anomaly detection by considering healthy patients as normal data. At training time an anomaly detector is trained on a training set containing only healthy patients. At test time, a new patient is evaluated and the model is required to predict whether the patient is normal (healthy) or anomalous. The evaluation metric is the area under the ROC curve (ROC-AUC) between the model predictions and the ground truth healthy/diabetes labels. For this illustrative example, we use a $k$NN anomaly detector. Despite its simplicity, this simple detector achieves state-of-the-art results on this dataset. 

The expressivity of the representation is critical to our approach. We use the number of dimensions $d$ as a proxy for the expressivity. We gradually increase the dimensionality of the representation $d$ by starting from a single attribute ($d=1$) and adding features (by relevancy) to the representation until it contains all the provided attributes ($d=8$). The results are provided in \cref{fig:motivating_example}. It is clear that anomaly detection performance increases as the number of attributes is increased from $d=1$ to $d=4$. However, with every further attribute that is added, results are degraded with the lowest level at $d=8$. We attribute the observed phenomenon to the sufficiency and over-expressivity of the representation. As four out of the eight attributes are necessary for separating patients with and without diabetes, $d=4$ is required for a sufficient representation. However, the remaining four attributes are a nuisance for the anomaly detection tasks making the representation over-expressive. When including irrelevant attributes in the representation, the $k$NN density estimator becomes less sensitive to the anomalies of interest. The presence of non-informative attributes reduces the discriminative power of the representation.

This motivating example highlights the delicate balance between representation sufficiency and over-expressivity in anomaly detection. It demonstrates that blindly increasing representation expressivity does not guarantee improved detection performance. Instead, it is crucial to navigate the trade-off between representation complexity and sensitivity to anomalies within specific attributes. This requires priors on the nature of the anomalies, which requires some form of guidance that can be implicit or explicit.

\section{Bias-Variance Tradeoff}

In this section, we use a simple toy model to demonstrate the tradeoff between sufficiency and over-expressivity. Specifically, it shows that increasing the number of dimensions used by an anomaly detector will eventually reduce the anomaly detection capabilities. We show that as the number of dimensions $d$ tends to infinity, the sensitivity (true positive rate) approaches the false positive rate. As random detectors also achieve sensitivity equal to the false positive rate this constitutes a complete failure of the detector.  We focus on anomaly detection by density estimation as the leading methods for both image anomaly detection and segmentation tasks are based on density estimation \cite{panda_,droc,patchcore,transformaly,density2_cvpr23_contrastive,density4_eccv22_contrastive}.

\subsection{A toy model}

We model the normal data by a multivariate Gaussian distribution of $d$ dimensions with an identity covariance matrix. We further assume that the anomalous data distribution is identical to the normal data except for a shift in the first axis. There is no loss of generality as the distributions are spherically symmetric, and therefore any offset between the means of the distributions can be rotated to align with the first axis. We show that following the result holds: 

\begin{theorem}
\label{theorem1}
Let the population distribution of samples ($D$) be given by: 

\begin{equation}
D=(1-\pi)\cdot\mathcal{N}(0,I) + \pi\cdot\mathcal{N}(\Delta\cdot e_1, I)
\end{equation}

where $\pi$ is a Bernoulli distribution describing the probability of a given sample to be anomalous.
We assume an unknown ground truth anomaly labeling function $l:D\to\{0,1\}$. We compute the anomaly score for test samples using their likelihood with respect to the normal data distribution: $s_d(x)=\frac{e^{-\frac{1}{2}||x||^2}}{{\sqrt{2\pi}}^{d}}$. Samples are classified as anomalous if their likelihood $s_d(x)$ is lower than the threshold $t_{d}$. We prove the following statements hold:

\begin{enumerate}
    \item 
    The true positive rate converges to the false positive rate as the dimension $d$ tends to infinity.

    \item For large values of $d$, the difference between the true positive rate and the false positive rate decays as $\frac{1}{\sqrt{d}}$.
\end{enumerate}
\end{theorem}

\begin{proof}[Proof Sketch of Theorem \ref{theorem1}]
We begin with the following lemma.
\begin{lem}
\label{lem_1}
 Let $t_{d}(y) \equiv \frac{e^{-\frac{1}{2}(y\sqrt{2d}+d)}}{\sqrt{2\pi}^d}$ where $y\in\mathbb{R}$ parameterized the sensitivity of the method. For a sufficiently large $d$ for the central limit theorem to hold, the false positive rate of the method converges to a constant.
\end{lem}

The proof of Lemma.~\ref{lem_1} is provided in App.~\ref{app:lemma_proof}. \hfill

Having set the threshold as $t_{d}(y)$, the false positive rate is fixed; we compute the dependence of the true positive rate on the dimension $d$.

To calculate the true positive rate, we need to compute the distribution of squared norms $\|x\|^2$ of anomalous examples. By \cite{chi_square_proof}, we have that for large $d$ the distribution is:

\begin{gather*}
     \|x\|^2 \sim \mathcal{N}(2d+\Delta^2, 4\Delta^2+2d)
\end{gather*}

In App.~\ref{app:thm_proof}, we derive an expression for the true positive rate:
\begin{gather*}
    TPR(t_{d}) = Pr( \|x\|^2 > y\sqrt{2d}+d ~|~ l(x)=1) = \ldots
    \\ = \int_{y\sqrt{2d}+d}^{\infty} \frac{1}{\sqrt{2\pi\cdot(4\Delta^2 + 2d)}} \cdot e^{-\frac{1}{2}\cdot(\frac{y' - (d + \Delta^2)}{\sqrt{4\Delta^2 + 2d}})^2}dy' \underset{\lim_{d\to\infty}}{\approx} FPR(t_{d}) +  
\Delta^2 \cdot \frac{1}{\sqrt{2\pi\cdot2d}} \cdot e^{-\frac{1}{2}\cdot y^{2}} 
\end{gather*}

 We see that the true positive rate decomposes into two terms: the false positive rate and another factor converging to $0$ at a rate of $\frac{1}{\sqrt{d}}$. As equal true positive and false positive rates indicate a random detector, the result describes convergence toward random predictions.

\end{proof}

\subsection{No Free Lunch}

Our theoretical analysis reveals a fundamental trade-off between representation sufficiency and over-expressivity in anomaly detection. While it is tempting to believe that increasing the expressivity of representations will lead to continuous improvements in anomaly detection performance, we provide evidence to the contrary. Building upon these findings, we introduce the concept of the \textit{"no free lunch"} principle in anomaly detection: including an ever-increasing number of attributes deteriorates our ability to detect anomalies in each of them. Therefore, trying to build a detector that would catch all types of anomalies is impossible, and detection must rely on some prior assumption regarding the nature of the anomalies we look for.
In other words, there is no universal anomaly detection method that excels in all scenarios without explicit guidance. 

This finding prompts us to explore the practical implications of over-expressivity in state-of-the-art representations. Thus, in what following sections, we conduct an extensive empirical investigation to demonstrate how current anomaly detection methods are affected by this "no free lunch" principle. By evaluating these methods across diverse datasets and anomaly types, we aim to showcase the real-world implications of over-expressive representations and shed light on the limitations of existing approaches.

\section{Empirical Investigation - Quantifying the effect of over-expressivity}
\label{sec:exp}

We conducted an extensive empirical study to investigate the impact of over-expressivity and the limitations of existing anomaly detection methods. We select a range of datasets to represent a wide range of anomalies and evaluate the generalization of anomaly detection approaches beyond novel object class anomalies.

\textbf{Datasets.} The following datasets were used in our analysis: \textit{Cars3D} \cite{cars3d}: This dataset consists of 17,568 3D car models with variations in elevation, rotation (azimuth), and car model (object). \textit{RaFD} \cite{rafd}: The Radboud Faces Database (RaFD) contains 8,040 facial images of different individuals expressing various emotions. Each image is annotated with emotional attributes such as happiness, anger, sadness, etc as well as person identity and head pose (angle). \textit{CelebA} \cite{celeba}: The CelebA dataset consists of celebrity face images with annotated attributes. In \cref{fig:data} we present a visualization of the dataset samples and the attributes we used for our evaluations. For each dataset, we split the data into training and testing sets with an 85:15 ratio. We trained the anomaly detection models on the training set and evaluated their performance on the testing set. Full training and implementation details are in App.~\ref{app:exp}.

\textbf{Baselines.} We conducted evaluations on prominent self-supervised and pretrained feature adaptation anomaly detection methods, along with an out-of-distribution (OOD) method \cite{ood_balaji} that incorporates guidance on the relevant attribute. The baselines we considered include DN2 with unadapted ImageNet pretrained ResNet features, PANDA \cite{panda_}, MSAD \cite{mean_shifted}, pretrained DINO \cite{dino, reiss_workshop}, and CSI \cite{csi}. To explore the impact of using guidance, we provided attribute values during the training phase for the OOD method. Specifically, we followed the state-of-the-art OOD method \cite{ood_balaji} which finetune a pretrained backbone using multi-class classification objective (provided by attribute values guidance information). It is important to note that the model was trained exclusively on normal data, like the other baselines. All methods are applied with their default hyperparameters as given in their official implementations. Unless otherwise specified, they use the same ResNet18 \cite{resnet} backbone architecture (except DINO, which utilizes a ResNet50 architecture).

\begin{figure}[t]
  \begin{center}
    \begin{tabular}{c}
    \includegraphics[scale=0.9875]{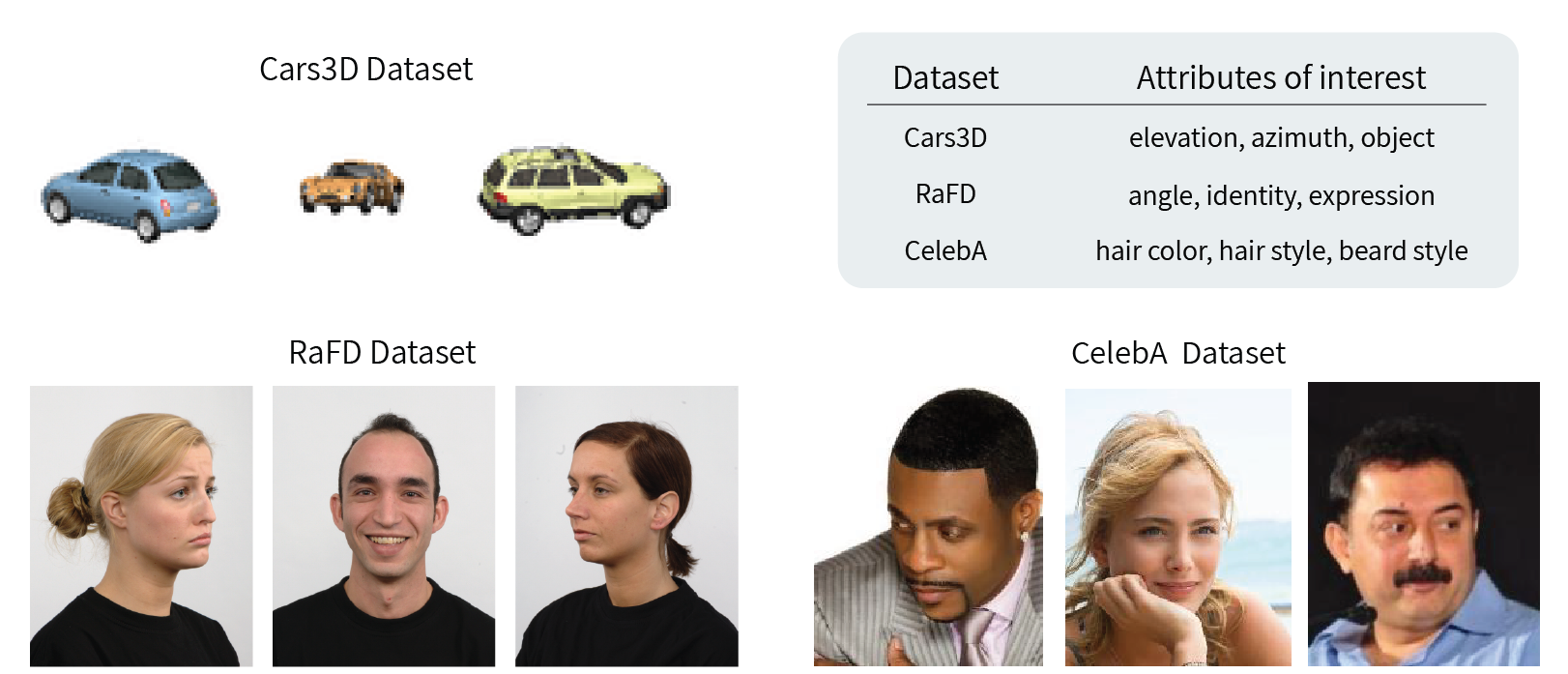} 
    \end{tabular}
  \end{center}
    \caption{Representative images of the different datasets, and their attributes of interest.}
    \label{fig:data}
\end{figure}

\subsection{Experimental Setting}
To evaluate the efficacy of each method, we conducted experiments among many datasets and settings: multi-value normal data and one-class classification. This comprehensive approach allows us to capture the overall performance trends and provide a comprehensive analysis. We report the average performance (ROC-AUC metric) obtained in each evaluation, providing a robust assessment of the models' capabilities. We used a standard $k$-nearest neighbors ($k$NN) density estimator for anomaly scoring. To assess the over-expressivity of our baselines, we adopted a supervised linear probing approach on the \textit{representation already learned} from the anomaly detection task. This involved augmenting the training set with anomalous samples and training a linear classifier on the baselines' \textit{representations} to differentiate between normal and anomalous instances. The resulting logits from the linear probing experiment were used as the anomaly scores. The same test sets were used for both the $k$NN density estimator and linear probing anomaly scoring. In each experiment, we evaluated the performance of the anomaly detection models across the first three values of each attribute within the given dataset. The implementation details can be found in App.~\ref{app:exp}.

\textbf{Single-value and multi-value classification settings.}
In the single-value setting, the normal data consists of a single semantic class (one-class classification), and no labels are provided. For each experiment, we designated a single value as \textit{normal} while considering all other values as \textit{anomalous}. We further assessed the performance of the methods in the multi-value setting as introduced by \cite{ahmed_courville}. In this setting, the normal data consists of multiple semantic classes, and no labels are provided. This setting poses a greater challenge compared to the single-value setting, as the normal data exhibits a multi-modal distribution. For each experiment conducted in this setting, we designated a single value as \textit{anomalous} while considering all other values as \textit{normal}. In both settings, the representations were trained solely on the normal samples.

\begin{table}[t]
  \caption{Single-value anomaly detection performance (mean ROC-AUC \%). $k$NN best in \textbf{bold black} and linear probing best in \textbf{\textcolor{blue}{bold blue}}.}
  \label{tab:single_value}
  \centering
  \begin{tabular}{llcccccccccc}
    \toprule
     \multirow{2}{*}{\rotatebox{90}{Data}} & Method & \multicolumn{2}{c}{CSI} & \multicolumn{2}{c}{DN2} & \multicolumn{2}{c}{PANDA} &  \multicolumn{2}{c}{MSAD} &  \multicolumn{2}{c}{DINO} \\
    \cmidrule(r){2-2}
    \cmidrule(r){3-4}
    \cmidrule(r){5-6}
    \cmidrule(r){7-8}
    \cmidrule(r){9-10}
    \cmidrule(r){11-12} 
     & Relevant attr. & $k$NN & Lin. &  $k$NN & Lin. &  $k$NN & Lin. &  $k$NN & Lin. &  $k$NN & Lin. \\
     \cmidrule(r){2-2}
    \cmidrule(r){3-4}
    \cmidrule(r){5-6}
    \cmidrule(r){7-8}
    \cmidrule(r){9-10}
    \cmidrule(r){11-12} 
    \multirow{4}{*}{\rotatebox{90}{Cars3D}} & Elevation & 69.8 & 88.2 & 79.5 & 88.6 & \textbf{80.2} & 88.4 & 79.5 & \textbf{\textcolor{blue}{89.0}} & 77.4 & 83.7\\
    & Azimuth & 73.1 & 93.6 & 89.2 & \textbf{\textcolor{blue}{98.0}} & 90.4 & 97.9 & 88.6 & \textbf{\textcolor{blue}{98.0}} & \textbf{91.1} & 97.3\\
    & Object & 91.7 & 86.5 & 99.7 & \textbf{\textcolor{blue}{99.9}} & 99.7 & \textbf{\textcolor{blue}{99.9}} & 99.7 & \textbf{\textcolor{blue}{99.9}} & \textbf{99.9} & \textbf{\textcolor{blue}{99.9}}\\
    & Average & 78.2 & 89.4 & 89.5 & 95.5 & \textbf{90.1} & 95.5 & 89.3 & \textbf{\textcolor{blue}{95.7}} & 89.5 & 93.5\\
    \midrule
    \multirow{4}{*}{\rotatebox{90}{RaFD}} & Angle & 89.2 & 95.2 & 98.2 & \textbf{\textcolor{blue}{100}} & 98.9 & \textbf{\textcolor{blue}{100}} & 99.5 & 99.9 & \textbf{99.9} & \textbf{\textcolor{blue}{100}}\\
    & Identity & 83.2 & 85.2 & 98.2 & 99.9 & 98.2 & 99.9 & \textbf{100} & \textbf{\textcolor{blue}{100}} & \textbf{100} & \textbf{\textcolor{blue}{100}} \\
    & Expression & 53.8 & 66.0 & 57.2 & 82.3 & 57.2 & 81.1 & 56.7 & 83.1 & \textbf{60.0} & \textbf{\textcolor{blue}{89.2}} \\
    & Average & 75.4 & 82.1 & 84.5 & 94.1 & 84.8 & 93.7 & 85.4 & 94.3 & \textbf{86.7} & \textbf{\textcolor{blue}{96.1}} \\
    \midrule
    \multirow{4}{*}{\rotatebox{90}{CelebA}} & Hair color & 50.2 & 80.9 & 56.9 & 84.5 & 57.0 & 82.0 & \textbf{59.7} & 85.3 & 58.5 & \textbf{\textcolor{blue}{87.1}} \\
    & Hair style & 51.1 & 79.5 & 54.4 & 81.6 & 54.5 & 79.9 & \textbf{56.6} & 82.4 & 55.3 & \textbf{\textcolor{blue}{84.6}} \\
    & Beard style & 50.8 & 87.8 & 57.6 & 91.5 & 58.9 & 90.8 & \textbf{62.6} & 91.1 & 60.6 & \textbf{\textcolor{blue}{94.2}} \\
    & Average & 50.7 & 82.7 & 56.3 & 85.9 & 56.8 & 84.2 & \textbf{59.6} & 86.3 & 58.1 & \textbf{\textcolor{blue}{88.7}}\\
    \bottomrule
  \end{tabular}
\end{table}

\subsection{Main Results}
\label{sec:main_results}
We present the evaluation results for the single-value setting in \cref{tab:single_value}. Our findings reveal inconsistencies in anomaly detection across different attributes of interest. Surprisingly, even in the case of the synthetic Cars3D dataset, where we initially expected the baselines to perform very well, we observe low accuracy in detecting elevation and azimuth anomalies. We found a notable difference between $k$NN and linear probing performance for each method. This gap emphasizes the over-expressivity of the representations, as the $k$NN estimator becomes less sensitive to the anomalies of interest. Linear classifier achieves higher anomaly detection accuracy as it can learn to focus on the relevant attributes. The results on Cars3D underscore the over-expressivity of the representation and indicate that very expressive representations fail to detect simple anomalies when evaluated beyond the well-studied object-centric datasets. In the RaFD dataset, our baselines demonstrate strong performance in detecting anomalies related to the angle and identity attributes. However, they struggle when confronted with anomalies associated with the expression attribute. Similarly, in the CelebA dataset, we encounter an even more challenging scenario where all baselines exhibit poor performance across all attributes. In some cases, the limited effectiveness of the linear probing approach indicates the representations are not sufficiently expressive, as they fail to sufficiently encode the relevant attribute information. Consequently, the representation cannot effectively differentiate between normal and anomalous data.

\cref{tab:multi_value} presents the results of the multi-value setting, where we conducted additional evaluations of our baselines and observed similar trends to those in the single-value setting. However, the multi-value setting provided an opportunity to investigate the out-of-distribution (OOD) baseline that incorporates attribute-specific guidance. Our results strongly support the notion that prior knowledge of the relevant attributes significantly aids in discriminating anomalies. Although the performance is still imperfect, OOD outperformed the other baselines. Moreover, the linear probing results demonstrate nearly perfect performance for all attributes in Cars3D and RaFD, surpassing other baselines by a large margin. This highlights the over-expressivity of the representation, particularly in the context of the OOD baseline, and underscores the importance of incorporating guidance regarding the relevant attributes.
Lastly, our evaluation demonstrates the notable performance advantage of pretrained methods compared to self-supervised methods in the context of anomaly detection. Specifically, the pretrained baselines consistently outperform CSI, a representative self-supervised method, which aligns with previous research findings in the field.

\subsection{Analysis of Guidance in Leading Anomaly Detection Algorithms}
\label{sec:guidance_analysis}

Having established that representations must be sufficient but not over-expressive, it is clear that guidance must be used to construct appropriate representations. While anomaly detection methods strive to use minimal guidance, nearly all deep anomaly detection methods use some form or another. Here, we analyze the top-performing anomaly detection paradigms through the lens of guidance for reducing representation over-expressivity while preserving the relevant information.

\textbf{Self-supervised learning (SSL) methods.} SSL methods \cite{geom,csi,droc,ssl1_eccv22,density4_eccv22_contrastive,density2_cvpr23_contrastive} have gained popularity in anomaly detection due to their ability to reduce labeling burden and impose weaker priors on expected anomalies. However, SSL methods still provide quite strong priors on the expected anomalies. Contrastive methods, which are the top-performing SSL methods, require representations to be invariant to well-designed augmentations (e.g. horizontal reflections or translations). While this reduces the number of degrees of freedom in the representation $d$, reducing over-expressivity, it poses a challenge when the anomaly lies within the invariant factors, such as an unusual horizontal flip or image translation. Many SSL methods also use negative augmentations to guide the representation to specified relevant attributes, e.g., rotations in contrastive methods \cite{csi,droc} or generative models in anomaly segmentation (e.g. \cite{li2021cutpaste,contrative_ssl2_eccv22}). While this increases representation sufficiency, the choice of augmentations must align with the characteristics of the (unknown) anomalies.

\begin{table}[t]
  \caption{Multi-value anomaly detection performance (mean ROC-AUC \%). $k$NN best in \textbf{bold black} and linear probing best in \textbf{\textcolor{blue}{bold blue}}.}
  \label{tab:multi_value}
  \centering
  \begin{tabular}{llcccccccccc}
    \toprule
     \multirow{2}{*}{\rotatebox{90}{Data}} & Method & \multicolumn{2}{c}{CSI} & \multicolumn{2}{c}{DN2} & \multicolumn{2}{c}{PANDA} &  \multicolumn{2}{c}{MSAD} &  \multicolumn{2}{c}{OOD} \\
    \cmidrule(r){2-2}
    \cmidrule(r){3-4}
    \cmidrule(r){5-6}
    \cmidrule(r){7-8}
    \cmidrule(r){9-10}
    \cmidrule(r){11-12} 
     & Relevant attr. & $k$NN & Lin. &  $k$NN & Lin. &  $k$NN & Lin. &  $k$NN & Lin. &  $k$NN & Lin. \\
     \cmidrule(r){2-2}
    \cmidrule(r){3-4}
    \cmidrule(r){5-6}
    \cmidrule(r){7-8}
    \cmidrule(r){9-10}
    \cmidrule(r){11-12} 
    \multirow{4}{*}{\rotatebox{90}{Cars3D}} & Elevation & 67.6 & 99.1 & 58.5 & 88.6 & 51.8 & 87.6 & 59.2 & 89.3 & \textbf{69.1} & \textbf{\textcolor{blue}{99.2}}\\
    & Azimuth & 42.2 & 97.4 & 82.4 & 98.0 & 79.8 & 98.5 & 81.8 & 98.4 & \textbf{84.1} & \textbf{\textcolor{blue}{99.9}}\\
    & Object & 75.0 & 99.7 & 87.9 & \textbf{\textcolor{blue}{99.9}} & 78.1 & 98.6 & 85.1 & \textbf{\textcolor{blue}{99.9}} & \textbf{96.8} & \textbf{\textcolor{blue}{99.9}}\\
    & Average & 61.6 & 98.7 & 76.2 & 95.5 & 69.9 & 94.9 & 75.3 & 95.9 & \textbf{83.3} & \textbf{\textcolor{blue}{99.7}}\\
    \midrule
    \multirow{4}{*}{\rotatebox{90}{RaFD}} & Angle & 88.0 & 98.5 & 97.7 & \textbf{\textcolor{blue}{100}} & 88.5 & \textbf{\textcolor{blue}{100}} & 98.5 & \textbf{\textcolor{blue}{100}} & \textbf{100} & \textbf{\textcolor{blue}{100}} \\
    & Identity & 66.9 & 85.4 & 97.5 & 99.9 & 93.4 & 99.7 & 99.3 & \textbf{\textcolor{blue}{100}} &\textbf{99.9} & \textbf{\textcolor{blue}{100}}\\
    & Expression & 49.7 & 64.5 & 54.7 & 82.2 & 53.1 & 81.8 & 54.8 & 83.1 & \textbf{70.8} & \textbf{\textcolor{blue}{97.9}}\\
    & Average & 68.2 & 82.8 & 83.3 & 94.1 & 78.3 & 93.8 & 84.2 & 94.4 & \textbf{90.2} & \textbf{\textcolor{blue}{99.3}}\\
    \midrule
    \multirow{4}{*}{\rotatebox{90}{CelebA}} & Hair color & 53.2 & 84.1 & 50.1 & 84.5 & 50.6 & 84.5 & 50.2 & 85.4 & \textbf{58.6} & \textbf{\textcolor{blue}{88.3}} \\
    & Hair style & 51.0 & 79.7 & 51.4 & 81.6 & 52.1 & 77.7 & 51.8 & 82.5 & \textbf{55.5} & \textbf{\textcolor{blue}{83.7}} \\
    & Beard style & 53.9 & 91.4 & 61.0 & 91.5 & 63.1 & 89.4 & 61.0 & 91.2 & \textbf{71.7} & \textbf{\textcolor{blue}{95.5}}\\
    & Average & 52.7 & 85.0 & 54.2 & 85.9 & 55.2 & 82.0 & 54.3 & 86.4 & \textbf{61.9} & \textbf{\textcolor{blue}{89.2}}\\
    \bottomrule
  \end{tabular}
\end{table}

\textbf{Pretrained representations.} Recently, anomaly detection methods based on pretrained representations, particularly those based on ImageNet or CLIP, have emerged as dominant approaches in image anomaly detection and segmentation benchmarks \cite{panda_,patchcore}. These pretrained representations have significantly increased the expressivity of the models, as shown in \cref{tab:single_value}. By leveraging large-scale pretraining on massive datasets, these representations capture a wide range of visual patterns and features, enhancing their potential to be sufficient for anomaly detection tasks. However, it is important to note that the increase in expressivity is not uniform across all aspects of anomalies. ImageNet-pretrained representations, for example, exhibit a strong bias towards object categories rather facial than expressions. As a result, they are more adept at detecting anomalies related to object categories rather than anomalies associated with expression variations (as can be seen in \cref{sec:main_results}). This selective expressivity illustrates the inherent tradeoff in using guidance for anomaly detection. While pretrained representations can help reduce over-expressivity and improve performance on certain types of anomalies, they also introduce the risk of insufficient representations when the guidance is not aligned with the anomalies of interest.

\textbf{Out-of-distribution (OOD) detection.} In OOD, the user selects a specific factor, such as an object class, that is considered crucial for identifying anomalies. The values of this factor in the normal data are labeled. In this approach, a sample exhibiting an unusual value of the specified attribute, such as belonging to a novel class, is labeled as anomalous. It is important to note that samples with highly unusual values in other attributes may not be classified as anomalous in this paradigm. Mathematically, this corresponds to choosing a representation with a $d=1$ dimensionality. By selecting a single attribute as the basis for anomaly detection, the method achieves high accuracy if the guidance is correct and the chosen attribute is indeed the relevant one for identifying anomalies. However, if the user's guidance attribute does not align with the true relevant attribute, the insufficiency of the representation significantly reduces the accuracy of the anomaly detection process.
We provide empirical support for this claim in \cref{tab:ood}. We see that relying on an irrelevant attribute for guidance in OOD detection results in significantly reduced accuracy (see App.~\ref{app:exp} for details). This highlights the "no free lunch" aspect of anomaly detection representations.

\begin{table}[t]
  \caption{OOD multi-value anomaly detection performance (mean ROC-AUC \%) when guidance attribute is irrelevant. $k$NN best in \textbf{bold black} and linear probing best in \textbf{\textcolor{blue}{bold blue}}.}
  \label{tab:ood}
  \centering
  \begin{tabular}{lllcccc}
    \toprule
     \multirow{2}{*}{\rotatebox{90}{Data}} & Method & & \multicolumn{2}{c}{OOD with irrelevant attribute} &  \multicolumn{2}{c}{OOD} \\
    \cmidrule(r){2-2}
    \cmidrule(r){3-3}
    \cmidrule(r){4-5}
    \cmidrule(r){6-7}
     & Relevant attribute & Guidance attribute & $k$NN & Lin. & $k$NN & Lin. \\
     \cmidrule(r){2-2}
     \cmidrule(r){3-3}
    \cmidrule(r){4-5}
    \cmidrule(r){6-7}
    \multirow{4}{*}{\rotatebox{90}{Cars3D}} & Elevation & Azimuth & 55.6 & 78.7 & \textbf{69.1} & \textbf{\textcolor{blue}{99.2}}\\
    & Azimuth & Object & 70.3 & 95.3 & \textbf{84.1} & \textbf{\textcolor{blue}{99.9}}\\
    & Object & Azimuth & 67.1 & 98.0 & \textbf{96.8} & \textbf{\textcolor{blue}{99.9}}\\
    & Average &  & 64.3 & 90.7 & \textbf{83.3} & \textbf{\textcolor{blue}{99.7}}\\
    \midrule
    \multirow{4}{*}{\rotatebox{90}{RaFD}} & Angle & Expression & 99.8 & \textbf{\textcolor{blue}{100}} & \textbf{100} & \textbf{\textcolor{blue}{100}} \\
    & Identity & Angle & 72.8 & \textbf{\textcolor{blue}{100}} & \textbf{99.9} & \textbf{\textcolor{blue}{100}}\\
    & Expression & Identity & 55.3 & 84.5 & \textbf{70.8} & \textbf{\textcolor{blue}{97.9}}\\
    & Average & & 76.0 & 94.8 & \textbf{90.2} & \textbf{\textcolor{blue}{99.3}}\\
    \midrule
    \multirow{4}{*}{\rotatebox{90}{CelebA}} & Hair color & Beard style & 54.2 & 84.9 & \textbf{58.6} & \textbf{\textcolor{blue}{88.3}} \\
    & Hair style & Hair color &52.4 & 80.0 & \textbf{55.5} & \textbf{\textcolor{blue}{83.7}} \\
    & Beard style & Hair style & 64.9 & 90.4 & \textbf{71.7} & \textbf{\textcolor{blue}{95.5}}\\
    & Average & & 57.2 & 85.1 & \textbf{61.9} & \textbf{\textcolor{blue}{89.2}}\\
    \bottomrule
  \end{tabular}
\end{table}

\textbf{Outlier exposure (OE).} In this setting, the user specifies a training set of samples that they assume to be more similar to the anomalies than normal data. The anomaly detection task then becomes supervised, similar to our linear probing setting. Unlike density estimation approaches, OE methods are not penalized for over-expressivity in their representations, as the method can learn to ignore nuisance factors. On the other hand, OE methods heavily rely on the similarity between the simulated anomalies (OE data) and the true anomalies that the model will encounter during testing. When the OE data accurately captures the characteristics of the anomalies, model performance is improved. On the other hand, if the OE data is misspecified or does not adequately represent the true anomalies, model performance is reduced.

\textbf{Nuisance disentangled representations.} This class of methods (e.g., \cite{red_panda}, \cite{zhang2023robustness}) directly reduces over-expressivity by explicitly removing factors of variation which the user specifies as a nuisance. When the guidance is correct, the reduced over-expressivity results in improved performance. As in the SSL case. When a relevant attribute is labeled as a nuisance, the representation runs the risk of being insufficient.

\section{Limitations}

\textbf{Beyond the toy model.} The simplicity of our toy model allowed us to derive fundamental insights, but real-world normal and anomalous data are often governed by more complex probability distributions and models. We hypothesize that the "no free lunch" principle remains valid for more complex distributions. This investigation is left for future work.

\textbf{Weak and few-shot supervision.} One way to avoid the over-expressivity issue is to use labeled anomalies. The labels can come in different forms e.g., labeling just a few samples, or labeling bags of samples. These techniques can be analyzed through the framework of supervised learning, instead of our framework.

\textbf{Guidance in the real world.} Designing practical guidance methods for learning representations in real-world scenarios remains an open question. Future research must take into account factors such as the availability and reliability of prior knowledge and the ease of providing guidance.

\textbf{Non density estimation methods.} We have not analyzed methods here such as auto-encoders \cite{venkataramanan2020attention} or GAN-based methods \cite{akcay2019ganomaly}. This is justified as density estimation methods \cite{panda_,droc, patchcore} typically perform better than the other techniques. We leave an analysis of the "no free lunch" principle for other anomaly detection frameworks for future work.

\section{Conclusion}
In this paper, we elucidate the performance degradation of anomaly detection methods due to over-expressive representations. Our investigation revealed a fundamental trade-off between representation sufficiency and over-expressivity, highlighting the limitations of solely relying on the increased scale and expressivity of deep networks. Through a novel theoretical toy model, we have demonstrated that increasing representation expressivity does not guarantee improved anomaly detection performance. Our model establishes a "no free lunch" principle for anomaly detection, stating that increased expressivity will eventually lead to performance degradation. Furthermore, we extensively investigated state-of-the-art representations and found that they often suffer from over-expressivity, leading to failure to detect various types of anomalies other than object categories.

\section{Acknowledgements}

Tal Reiss and Niv Cohen are partially supported by funding from the Israeli Science Foundation. Niv Cohen is also partially supported by funding from the Center of Data Science.

{\small

\bibliographystyle{unsrt}
\bibliography{cite}
}

\newpage
\clearpage
\appendix
\centerline{\textbf{\LARGE Appendix}}
\vspace{1.5em}
\centerline{\textbf{\Large No Free Lunch: The Hazards of Over-Expressive}}
\centerline{\textbf{\Large Representations in Anomaly Detection}}
\vspace{1em}
\section{Proof of Lemma \ref{lem_1}}
\label{app:lemma_proof}
\begin{lem_1}
 Let $t_{d}(y) \equiv \frac{e^{-\frac{1}{2}(y\sqrt{2d}+d)}}{\sqrt{2\pi}^d}$ where $y\in\mathbb{R}$ parameterizes the sensitivity of the method. For $d$ sufficiently large  for the central limit theorem to hold, the false positive rate of the method converges to a constant.
\end{lem_1}
\begin{proof}[Proof of Lemma \ref{lem_1}]

The relation between the false positive rate (FPR) and the parameter $y$ is given by:
\begin{gather*}
      FPR(t_{d}) = Pr(s_{d}(x) < t_{d}~|l(x)=0) 
\end{gather*}  
Recall that for the normal data $\log(s_d(x)) = -\sum_{i=1}^d \frac{1}{2} x_i^2 - \frac{d}{2} log(2 \pi)$, and that $\log(t_d) = -\frac{1}{2}(y\sqrt{2d}+d) - \frac{d}{2} log(2 \pi)$. Therefore, the expression can be simplified as:
\begin{gather*}
     FPR(t_{d}) = Pr( \sum_{i=1}^{d} x_i^2 > y\sqrt{2d}+d) 
    \\ = Pr( \frac{\sum_{i=1}^{d} x_i^2 - d}{\sqrt{2d}} > y) 
     = Pr( \frac{Q_{d} - \mu_{d}}{\sigma_{d}} > y) 
\end{gather*}

where we denote $Q_d = \|x\|^2 = \sum_{i=1}^{d} x_i^2$, and $\{x_i\}_{i=1}^{d}$ are independent, standard normal random variables. Therefore, $Q_d$ is distributed according to the chi-squared distribution with $d$ degrees of freedom, i.e., $Q_d \sim \chi_{d}^2$. For moderately large values of $d$, as is the case with deep representations as explored in \cref{sec:exp},  the central limit theorem applies, and therefore for $d >> 1$ we have $\chi_{d}^2 \approx \mathcal{N}(d, 2d) = \mathcal{N}(\mu_d, \sigma_d^2)$.

We may now show that in the limit for large dimensions $d_1, d_2 \in \mathbb{N}$ the relation between $y$ and the false positive rate is independent of the dimensions:

\begin{gather*}
    FPR(t_{d_1}) = Pr( \frac{Q_{d_1} - d_1}{\sqrt{2d_1}} > y ~)   = Pr(  \mathcal{N}(0, 1) > y ~) 
    \\ = Pr( \frac{Q_{d_2} - d_2}{\sqrt{2d_2}} > y ) = FPR(t_{d_2})
\end{gather*}

The equality of both sides of the probability of a normal random variable being larger than $y$ follows from $(d,\sqrt{2d})$ is the mean and the standard deviation of the corresponding $Q_d$ function (in the central limit theorem).

This concludes the proof of the lemma.
\end{proof}

\section{Proof of Theorem \ref{theorem1}}
\label{app:thm_proof}
\begin{thm_1}
Let the population distribution of samples ($D$) be given by: 

\begin{equation}
D=(1-\pi)\cdot\mathcal{N}(0,I) + \pi\cdot\mathcal{N}(\Delta\cdot e_1, I)
\end{equation}

where $\pi$ is a Bernoulli distribution describing the probability of a given sample to be anomalous.
We assume an unknown ground truth anomaly labeling function $l:D\to\{0,1\}$. We compute the anomaly score for test samples using their likelihood with respect to the normal data distribution: $s_d(x)=\frac{e^{-\frac{1}{2}||x||^2}}{{\sqrt{2\pi}}^{d}}$. Samples are classified as anomalous if their likelihood $s_d(x)$ is lower than the threshold $t_{d}$. We prove the following statements hold:

\begin{enumerate}
    \item 
    The true positive rate converges to the false positive rate as the dimension $d$ tends to infinity.

    \item For large values of $d$, the difference between the true positive rate and the false positive rate decays as $\frac{1}{\sqrt{d}}$.
\end{enumerate}
\end{thm_1}

\begin{proof}[Proof of Theorem \ref{theorem1}]
We would like to describe the distance distribution of an anomalous sample $x$ to the center of the normal data distribution, $||x||^2=Q^{anom}_d$ to calculate their likelihood. Having the distribution of their likelihood (according to the normal data model), we can calculate their probability of being below the likelihood threshold $t_d(y)$.
We begin by showing that the distance distribution ($||x||^2$) of the anomalies is of a similar distribution to that of the normal data, with a modification driven by the parameter $\Delta$.
By \cite{chi_square_proof} we have that for large $d$ for the anomalous data:
\begin{gather*}
     Q^{anom}_d = \|x\|^2 = (x'_{1}+\Delta)^2 + \sum_{i=2}^{d} {x'}_{i}^{2}  \sim \mathcal{N}(2d+\Delta^2, 4\Delta^2+2d)
\end{gather*}
Where $x'$ describes the anomalous data distribution around its distribution center (shifted by $\Delta$ from the origin). Namely, $x_1 = x'_1+\Delta$, and $x_i = x'_i$ for $i \neq 1$.

Now it holds that:
\begin{gather*}
    TPR(t_d) = Pr(s_d(x) < t_d ~|~ l(x)=1) = Pr( Q^{anom}_d > y\sqrt{2d}+d ~)
    \\ = \int_{y\sqrt{2d}+d}^{\infty} \frac{1}{\sqrt{2\pi\cdot(4\Delta^2 + 2d)}} \cdot e^{-\frac{1}{2}\cdot(\frac{y' - (d + \Delta^2)}{\sqrt{4\Delta^2 + 2d}})^2}dy'
\end{gather*}

We make the substitution $\eta=\frac{\sqrt{2d}}{\sqrt{4\Delta^2+2d}}y'$ and get:
\begin{gather*}
    \int_{y\sqrt{2d}+d}^{\infty} \frac{1}{\sqrt{2\pi\cdot(4\Delta^2 + 2d)}} \cdot e^{-\frac{1}{2}\cdot(\frac{y' - (d + \Delta^2)}{\sqrt{4\Delta^2 + 2d}})^2}dy'
    \\ =    \int_{\frac{y\cdot2d + d\sqrt{2d}}{\sqrt{4\Delta^2+2d}}}^{\infty} \frac{1}{\sqrt{2\pi\cdot2d}} \cdot e^{-\frac{1}{2}\cdot(\frac{\eta\sqrt{\frac{4\Delta^2 + 2d}{2d}} - (d + \Delta^2)}{\sqrt{4\Delta^2 + 2d}})^2}d\eta
    \\ = \int_{\frac{y\cdot2d + d\sqrt{2d}}{\sqrt{4\Delta^2+2d}}}^{\infty} \frac{1}{\sqrt{2\pi\cdot2d}} \cdot e^{-\frac{1}{2}\cdot(\frac{\eta}{\sqrt{2d}} - \frac{(d + \Delta^2)}{{\sqrt{4\Delta^2 + 2d}}})^2}d\eta
\end{gather*}
We make another substitution $\eta'=\sqrt{2d}({\frac{\eta}{\sqrt{2d}}-\frac{d+\Delta^2}{\sqrt{4\Delta^2+2d}}+\frac{d}{\sqrt{2d}}})$ and get:
\begin{gather*}
    \int_{\frac{y\cdot2d + d\sqrt{2d}}{\sqrt{4\Delta^2+2d}}-\frac{\sqrt{2d}(d+\Delta^2)}{\sqrt{4\Delta^2+2d}}+d}^{\infty} \frac{1}{\sqrt{2\pi\cdot2d}} \cdot e^{-\frac{1}{2}\cdot(\frac{\eta'-d}{\sqrt{2d}})^{2}} d\eta' 
    \\ = \int_{\frac{y\cdot2d - \Delta^2\sqrt{2d}}{\sqrt{4\Delta^2+2d}}+d}^{\infty} \frac{1}{\sqrt{2\pi\cdot2d}} \cdot e^{-\frac{1}{2}\cdot(\frac{\eta'-d}{\sqrt{2d}})^{2}} d\eta' 
    \\ = \int_{\frac{y\cdot2d - \Delta^2\sqrt{2d}}{\sqrt{4\Delta^2+2d}}+d}^{y\sqrt{2d}+d} \frac{1}{\sqrt{2\pi\cdot2d}} \cdot e^{-\frac{1}{2}\cdot(\frac{\eta'-d}{\sqrt{2d}})^{2}} d\eta' + \int_{y\sqrt{2d}+d}^{\infty} \frac{1}{\sqrt{2\pi\cdot2d}} \cdot e^{-\frac{1}{2}\cdot(\frac{\eta'-d}{\sqrt{2d}})^{2}} d\eta' 
    \\ = \int_{\frac{y\cdot2d - \Delta^2\sqrt{2d}}{\sqrt{4\Delta^2+2d}}+d}^{y\sqrt{2d}+d} \frac{1}{\sqrt{2\pi\cdot2d}} \cdot e^{-\frac{1}{2}\cdot(\frac{\eta'-d}{\sqrt{2d}})^{2}} d\eta' + FPR(t_d) 
    \end{gather*}
    
Where the last transition follows from identifying the second part of the integral with the false positive rate we used in the proof of Lemma.~1. This rate is independent of $d$ in the limit.

We now bound the first integral by taking the maximum of the integral: 
    
    \begin{gather*}
    \leq (y\cdot\sqrt{2d} - \frac{y\cdot2d - \Delta^2\sqrt{2d}}{\sqrt{4\Delta^2+2d}}) \frac{1}{\sqrt{2\pi\cdot2d}} \cdot e^{-\frac{1}{2}\cdot(\frac{\frac{y\cdot2d - \Delta^2\sqrt{2d}}{\sqrt{4\Delta^2+2d}}}{\sqrt{2d}})^{2}} + FPR(t_d)
    \\ = (y\cdot\sqrt{2d} - \frac{y\cdot2d - \Delta^2\sqrt{2d}}{\sqrt{4\Delta^2+2d}}) \frac{1}{\sqrt{2\pi\cdot2d}} \cdot e^{-\frac{1}{2}\cdot(\frac{y\cdot\sqrt{2d} - \Delta^2}{\sqrt{4\Delta^2+2d}})^{2}} + FPR(t_d) 
\end{gather*}
Now, we note that for the exponent:
\begin{gather*}
    \lim_{d\to\infty} \frac{y\cdot\sqrt{2d} - \Delta^2}{\sqrt{4\Delta^2+2d}} = y \in \mathbb{R}
\end{gather*}
Moreover, we observe that:
\begin{gather*}
 \lim_{d\to\infty} (y\cdot\sqrt{2d} - \frac{y\cdot2d - \Delta^2\sqrt{2d}} {\sqrt{4\Delta^2+2d}}) 
 \\=  \lim_{d\to\infty} y(\sqrt{2d} - \frac{2d} {\sqrt{4\Delta^2+2d}}) + \lim_{d\to\infty}\frac{\Delta^2\sqrt{2d}}{\sqrt{4\Delta^2+2d}} = 0 + \Delta^2 = \Delta^2
\end{gather*}
Where, we used the fact that $\lim_{d\to\infty} y(\sqrt{2d} - \frac{2d} {\sqrt{4\Delta^2+2d}})=0$ and that $\lim_{d\to\infty}\frac{\Delta^2\sqrt{2d}}{\sqrt{4\Delta^2+2d}}=\Delta^2$, and therefore the limits laws applies. As a result, we get by the limits laws that:
\begin{gather*}
    \lim_{d\to\infty} (y\cdot\sqrt{2d} - \frac{y\cdot2d - \Delta^2\sqrt{2d}}{\sqrt{4\Delta^2+2d}}) \frac{1}{\sqrt{2\pi\cdot2d}} \cdot e^{-\frac{1}{2}\cdot(\frac{y\cdot\sqrt{2d} - \Delta^2}{\sqrt{4\Delta^2+2d}})^{2}}
    \\ = \lim_{d\to\infty} (y\cdot\sqrt{2d} - \frac{y\cdot2d - \Delta^2\sqrt{2d}}{\sqrt{4\Delta^2+2d}}) \cdot \lim_{d\to\infty} \frac{1}{\sqrt{2\pi\cdot2d}} \cdot \lim_{d\to\infty} \cdot e^{-\frac{1}{2}\cdot(\frac{y\cdot\sqrt{2d} - \Delta^2}{\sqrt{4\Delta^2+2d}})^{2}} 
    \\ = \Delta^2 \cdot 0 \cdot e^{-\frac{1}{2}\cdot y^{2}} = 0
\end{gather*}
where the limit that goes to $0$ is doing so at a rate of $\frac{1}{\sqrt{d}}$.
This ends the proof.
\end{proof}

\section{Experimental details \& More evaluations}
\label{app:exp}

\subsection{Dataset Descriptions}
In our analysis, we used three distinct datasets to evaluate our anomaly detection models. The first dataset, \textit{Cars3D} \cite{cars3d}, comprised a collection of 17,568 3D car models, showcasing variations in elevation, rotation (azimuth), and car model (object). The second dataset, \textit{RaFD} \cite{rafd}, featured 8,040 facial images capturing different individuals expressing a wide range of emotions. Each image in the RaFD dataset was annotated with emotional attributes such as happiness, anger, and sadness, along with person ID and head pose information. Lastly, we used the \textit{CelebA} dataset \cite{celeba}, which contained images of celebrity faces annotated with various attributes. For the CelebA dataset, we focused on hair color, hair style, and beard style attributes during our evaluation. These attributes were introduced by merging existing attributes in the dataset. The hair color attribute included \{\textit{Black\_Hair, Blond\_Hair, Brown\_Hair, Gray\_Hair\}} values, the hair style attribute included \{\textit{Bald, Bangs, Receding\_Hairline, Straight\_Hair, Wavy\_Hair\}} values, and the beard style attribute included \{\textit{Goatee, Mustache, No\_Beard\}} values. We utilized the CelebA dataset test set, which consisted of 39,829 images. This test set was subsequently split into a new train and test sets for our analyses. \cref{tab:attribute_splits} provides a summary of the attributes of interest for each dataset.

\subsection{Baselines}
In our comprehensive evaluation, we investigated various established baseline methods for anomaly detection, including both self-supervised and pretrained feature adaptation approaches. Furthermore, we incorporated an out-of-distribution (OOD) method \cite{ood_balaji} that utilizes guidance on the relevant attribute. The baselines we considered encompassed a diverse set of methods. Specifically, we evaluated DN2, which utilized unadapted ImageNet pretrained ResNet features, as well as top-performing pretrained representation adaptation methods such as PANDA \cite{panda_}, MSAD \cite{mean_shifted}, and pretrained DINO \cite{dino, reiss_workshop}. Additionally, we included CSI \cite{csi}, a leading self-supervised method that learns representations from scratch. By considering this comprehensive set of baselines, we aimed to provide a thorough analysis of different approaches for anomaly detection, spanning both self-supervised and pretrained feature adaptation paradigms.

\subsection{Implementation Details}
\textbf{Dataset splits.} For each dataset and relevant attribute, we employed an 85:15 train-test split. To evaluate our unguided anomaly detection approaches, we excluded anomalous samples from the train set. Instead, we focused solely on normal samples related to the specific attribute under investigation. Default hyperparameters were used for all methods, as specified in their official implementations.

\textbf{$k$NN density estimation.} For density estimation, we utilized the faiss \cite{faiss} library's implementation of the k-nearest neighbors ($k$NN) algorithm. To balance accuracy and computational efficiency, we set $k=10$.

\textbf{Linear probing.} To assess over-expressivity, we employed a linear probing approach. To evaluate linear probing, we augmented the training set by introducing anomalous samples to it. Specifically, we used the complete train set from the original split, incorporating its anomalous instances with attribute-specific modifications.

We trained a linear classifier to distinguish between normal and anomalous instances based on the \textit{representations} generated by the baselines. The training employed a stochastic gradient descent (SGD) optimizer with a binary cross-entropy loss function. The training lasted for 100 epochs, with an initial learning rate of 0.01 and a momentum factor of 0.9. During testing, the trained linear classifier was applied to compute logits which served as anomaly scores. Higher logits indicate a higher likelihood of an instance being anomalous. Importantly, the same test sets were used for both the $k$NN density estimator and linear probing anomaly scoring, ensuring a fair comparison between the two approaches.

\textbf{Out-of-distbution (OOD) detection.} We implemented an OOD method for incorporating guidance on the relevant attribute. Our approach was based on the state-of-the-art OOD method \cite{ood_balaji}, which focuses on fine-tuning a pretrained backbone using a multi-class classification (cross-entropy) objective by incorporating guidance through attribute annotations. During the training phase, we included attribute values as part of the input data to guide the model's learning process. By associating attribute values with normal instances, the model learns normal data patterns and characteristics. It is important to note that the model was trained exclusively on normal data, like the other baselines.

In each experimental setting, we evaluated the performance of the anomaly detection models across the first three values of each attribute within the given dataset.

\begin{table}
  \caption{Attribute splits for our benchmarks.}
  \label{tab:attribute_splits}
  \centering
  \begin{tabular}{ll}
    \toprule
    Dataset & Attributes of Interest\\
    \midrule
    Cars3D & elevation, azimuth, object \\
    RaFD & angle, identity, expression \\
    CelebA & hair color, hair style, beard style \\
    \bottomrule
  \end{tabular}
\end{table}

\subsection{OOD Detection with Irrelevant Attributes}
Our evaluation provides empirical evidence to demonstrate the impact of selecting an incorrect attribute for guidance in out-of-distribution (OOD) detection. To examine this, we conducted experiments and presented the results in \cref{tab:ood}. The table clearly illustrates that relying on an irrelevant attribute for guidance in OOD detection results in a noticeable decrease in accuracy. Specifically, in our experiment, we assigned the sequential attribute as the attribute for the irrelevant attribute OOD baseline (according to the order in \cref{tab:attribute_splits}). These findings emphasize the importance of selecting the appropriate attribute for effective OOD detection and highlight the "no free lunch" principle in anomaly detection representations.

\subsection{Training Resources}
Training each dataset presented in this paper takes approximately 3 hours on a single NVIDIA RTX-2080 TI.

\subsection{Licences}

\textbf{Code.} PANDA \cite{panda_} and MSAD \cite{mean_shifted} are licensed under a SOFTWARE RESEARCH LICENSE as described here\footnote{https://github.com/talreiss/Mean-Shifted-Anomaly-Detection/blob/main/LICENSE}. PyTorch uses a BSD-style license, as detailed in the license file\footnote{https://github.com/pytorch/pytorch/blob/master/LICENSE}, while faiss \cite{faiss} creates MIT licenses.

\textbf{Metrics.} For the ROC-AUC metric we use \textit{roc\_auc\_score} function from scikit-learn library.

\end{document}